\documentclass[letterpaper]{article} 
\usepackage{aaai25}  

\usepackage{times}  
\usepackage{helvet}  
\usepackage{courier}  
\usepackage[hyphens]{url}  
\usepackage{graphicx} 
\usepackage{natbib}  
\usepackage{caption} 
\usepackage{algorithm}
\usepackage{algorithmic}
\usepackage{amsthm}
\usepackage{multirow}

\usepackage{newfloat}
\usepackage{listings}
\DeclareCaptionStyle{ruled}{labelfont=normalfont,labelsep=colon,strut=off} 
\floatstyle{ruled}
\newfloat{listing}{tb}{lst}{}
\floatname{listing}{Listing}
\def\gt{\mathit{gt}}
\def\Pr{\mathbf{P}}

\def\legi{\alpha \in \gt}
\def\leg{\legi}
\def\lefi{\alpha \in f_i}
\def\lef{\lefi}
\def\leng{\alpha \notin \gt}
\def\d{D=d}
\def\c{c \in \m}
\def\notc{c \notin \m}
\def\m{M}

\def\lefj{\beta \in f_i}
\def\ci{c_\alpha \in \m}

\def\legj{\beta \in \gt}

\def\lefi{\alpha \in f_i}
\def\legi{\alpha \in \gt}

\newtheorem{definition}{Definition}[section]
\newtheorem{theorem}{Theorem}[section]
\newtheorem{corollary}{Corollary}[section]

\newtheorem{example}[theorem]{Example}

\def\toyota{\mathit{toyota}}
\def\ford{\mathit{ford}}
\def\dodge{\mathit{dodge}}
\def\us{\mathit{us}}
\def\japan{\mathit{japan}}
\def\fcar{f_\mathit{car}}
\def\car{\mathit{car}}

\title{Probabilistic Foundations for Metacognition via Hybrid-AI}
\author {
    Paulo Shakarian\textsuperscript{\rm 1},
    Gerardo I. Simari\textsuperscript{\rm 2},
    Nathaniel D. Bastian\textsuperscript{\rm 3}
}
\affiliations {
    \textsuperscript{\rm 1}Arizona State University, Tempe, AZ USA\\
    \textsuperscript{\rm 2} Departament of Computer Science and Engineering, Universidad Nacional del Sur (UNS) \& \\
    Institute for Computer Science and Engineering (UNS-CONICET), Bahía Blanca, Argentina\\
    \textsuperscript{\rm 2}United States Military Academy, West Point, NY USA\\

    pshak02@asu.edu, gis@cs.uns.edu.ar, nathaniel.bastian@westpoint.edu
}

\usepackage{bibentry}

\begin{document}

\maketitle

\begin{abstract}
Metacognition is the concept of reasoning about an agent's own internal processes, and it has recently received renewed attention with respect to artificial intelligence (AI) and, more specifically, machine learning systems.  This paper reviews a hybrid-AI approach known as ``error detecting and correcting rules'' (EDCR) that allows for the learning of rules to correct perceptual (e.g., neural) models.  Additionally, we introduce a probabilistic framework that adds rigor to prior empirical studies, and we use this framework to prove results on necessary and sufficient conditions for metacognitive improvement, as well as limits to the approach.  A set of future research directions is also provided.
\end{abstract}

%

\section{Introduction}

Originally a concept from developmental psychology~\cite{flavell1979metacognition}, metacognition refers to reasoning about an agent’s own internal processes~\cite{wei24}.  This idea of metacognition, considered the human brain's ``self-monitoring process''~\cite{demetriou1993architecture}, has been studied in a diverse set of fields~\cite{li2017applications,izzo2019survey,caesar2020nuscenes}.  Over the years, its study has been proposed in the field of AI~\cite{cox2011metareasoning,cox2005metacognition}, and recent interest~\cite{wei24} has reinvigorated this discussion.
In this paper, we review a recent, hybrid-AI style of metacognition known as ``error detection and correction rules'' (EDCR) that has been explored in a variety of models and use cases (see Table~\ref{tab:summary}). 
In short, the proposal is a hybrid-AI approach in which logical rules are learned to characterize the model performance of well-trained perceptual (e.g., neural) models.  This paper serves as a review of the results of these recent articles, and improves the underlying theoretical underpinnings by framing them in terms of a probabilistic argument, providing new insights into the use of hybrid-AI for metacognition.  
In particular, our novel theoretical framework allows us to explain some empirically observed results, as well as leads us to several new open research questions.

\section{Error Detection and Correction Rules}

In this framework, we are given a well-trained model, normally denoted as $f$, that takes some continuous input (e.g., a vector) and returns a set of class labels.  We will use subscripts (e.g., $f_i$ or $f_1$) to denote multiple models when relevant.  We note that the model need not be neural (though all existing studies have focused on neural models), and in general we do not assume access to the model weights.  We now introduce a small running example.

\begin{example}
\label{ex:1}
Let $f_{car}$ be a neural model trained on data set $D_{tng}$ such that given a sample (that we denote $x$) returns a subset of labels $\{\ford,\toyota,\dodge,\us,\japan\}$.
So, for example, for sample $\omega$, $f_{car}(\omega)=\{\dodge,\us\}$.
\end{example}

\begin{table}
\begin{center}
\begin{small}
    \begin{tabular}{|l|l|l|}
    \hline
        \textbf{Paper and} & \textbf{Base model}   & \textbf{Source of} \\
                                       \textbf{use case}&                 &\textbf{conditions} \\
        \hline
        \hline
        \citet{kri24} & ViT & Hier. class  \\
        Vision              &     & labels, other \\
        && models\\
        \hline
        \citet{xi24} &  CNN, LRCN, & Dom. kn, \\
         Trajectory classification & LRCNa & other models\\
        
        \hline
        \citet{lee24} &  CNN, RNN, & Other models\\
        Time series              classification & Attention     & \\
        
        \hline
    \end{tabular}
\end{small}
\caption{Summary of existing literature on EDCR.}  
\label{tab:summary}
\end{center}
\end{table}

The work on EDCR introduces the notion of a \textit{metacognitive condition}, which is some aspect of the environment, sensor, model, or metadata that may cause the model to be incorrect in a given circumstance.  In the existing work, the metacognitive condition (which we refer to as ``condition'', for short) has come from domain knowledge, model output (e.g,, labels at different levels of a hierarchy), or other models (e.g., different architectures, trained on different data) -- this is summarized in Table~\ref{tab:summary}.  
This leads us to the first type of rule, the error detecting rule, which is expressed in a first order logic syntax as follows:
\begin{eqnarray}
\mathit{error}^i_\alpha(X) \leftarrow \mathit{pred}_\alpha^i(X) \wedge \bigvee_{j\in \mathit{DC}_i}\mathit{cond}_j(X)
\end{eqnarray}
In words, this rule states that if any of the conditions in set $\textit{DC}_i$ are met, and model $i$ predicts class $\alpha$, then the model has an erroneous detection.  We note that we can replace the disjunction with a singleton and have multiple rules for a given model-class pair (i.e., using a logic program).  The second type of rule is a correction rule, which has the following syntax:
\begin{eqnarray}
\mathit{corr}_\beta^i(X) \leftarrow \bigvee_{j,\alpha\in \mathit{CC}_\beta}\left(\mathit{cond}_j(X) \wedge \mathit{pred}_\alpha^i(X) \right)
\end{eqnarray}
In words, this rule states that given a set of condition-class pairs relevant to class $\beta$ (denoted $\mathit{CC}_\beta$), for any of those pairs (e.g., model $i$ predicting $\alpha$ for the sample while condition $j$ is true), the rule dictates that the sample should be relabeled with label $\beta$.  
In both~\citet{xi24} and~\citet{lee24}, the detection step happens first (in both training and inference), allowing for samples to have some labels ``erased'' and only samples with erased labels will have new labels assigned by a correction rule.

\begin{example}
\label{ex:2}
Consider the model $f_{\mathit{car}}$ from Example~\ref{ex:1}.  Suppose we have a detection rule of the form:
\begin{eqnarray*}
\mathit{error}^\car_{\toyota}(X) \leftarrow \mathit{pred}_\toyota^\car(X) \wedge \mathit{cond}_{\us}(X)
\end{eqnarray*}
Here we have a single condition $\mathit{cond}_\us$, and we define it to be true for sample $x$ when $\us \in f_\car(x)$.
This is an example of using a class label from a different level of the hierarchy, as done in~\cite{kri24}.
Likewise, we can imagine a detection rule:
\begin{eqnarray*}
\mathit{corr}^\car_{\dodge}(X) \leftarrow \mathit{cond}_{\us}(X) \wedge \mathit{pred}_\toyota^\car(X)
\end{eqnarray*}
Here, if the condition-class pair $\mathit{cond}_\us$ and $\toyota$ are true, then the sample should be re-labeled as $\dodge$.
\end{example}

In Example~\ref{ex:2}, one could imagine having a priori knowledge that predicting $\japan$ and $\us$ together would yield an inconsistency.  
However, in such a case, it may make more sense to define the condition as the case where $\{\japan,\us\}\subseteq \fcar(x)$.  In fact, in prior work, the primary knowledge engineering task is to design the conditions, not the rules, which are learned from data, primarily using combinatorial techniques.
For example, in~\citet{xi24} detection rules are learned via submodular maximization of precision while constraining recall reduction,
while in~\citet{kri24} detection rules are learned via submodular ratio maximization to optimize toward F1.  
We also note that the detection rule of Example~\ref{ex:2} can also function as a constraint on class relationships, in this case meaning that $\us$ and $\toyota$ cannot go together.  This turns out to be useful, as the rules are already expressed in logic.  This enabled~\citet{kri24} to take learned detection rules and retrain a model with an LTN~\cite{ltn22} loss function to gain improved model performance.

\section{Probabilistic Interpretation and Results}

We now present a new probabilistic interpretation of EDCR.  First, we introduce a bit of notation: when the object in question is understood, we shall use the notation $f_i$ to refer to the set of propositions produced by model $f_i$ on the object.  So, for example, $\lefi$ means that model $f_i$ predicted label $\alpha$.  We shall also use the notation $\gt$ to denote a set of ground truth labels for the object; so, the statement $\lefi,\leng$ means that model $f_i$ incorrectly predicted class $\alpha$ for the object.  With this in mind, we can represent the precision and recall for a model as the following conditional probabilities.
\begin{eqnarray}
\textit{Precision:} & P_\alpha = \Pr(\legi \; | \; \lefi)\\
\textit{Recall:}    & R_\alpha = \Pr(\lefi \; | \; \legi)
\end{eqnarray}
We will also consider a random variable $D$ that specifies a specific distribution.  We can consider conditional probabilities with an unspecified distribution (as shown above) to be approximated from training data, while setting $D$ to be something specific would signify being out-of-distribution: e.g., the precision for model $i$ on class $\alpha$ under distribution~$d$ is written as $\Pr(\legi \; | \; \lefi,\d)$.
In what follows, we shall use $\m$ to denote the set of metacognitive conditions.  We note that this framework allows us to precisely define what it means for a condition to be error detecting, and introduce a new, formal definition to that effect.

\begin{definition}[Error Detecting]
\label{def:det}
A metacognitive condition~$c$ is \emph{error detecting} with respect to model $i$, class $\alpha$, and distribution $d$ if 
$\Pr(\legi \; | \; \lefi,\c,\d)$ is less than or equal to $\Pr(\legi \; | \; \lefi,\d)$.
\end{definition}

Intuitively, a condition is error detecting if the precision of the model drops when the condition is true.  The second property is distribution invariance.

\begin{definition}[Distribution Invariance]
A metacognitive condition $c$ is \emph{distribution invariant} with respect to model~$i$, class~$\alpha$, and set of distributions $\mathcal{D}$ if for any distributions $d \in \mathcal{D}$ it is error detecting with respect to that distribution.
\end{definition}

At first glance, distribution invariance seems to be a strong definition, and some may even think the set of conditions meeting that criteria would be small.  However, consider conditions based on constraints among classes; e.g., inconsistent assignments in a multi-class classification problem is one such case, and the configuration of the sensor is another.

With conditions in mind, we define precision and recall after applying a metacognitive condition:
\begin{eqnarray}
\textit{Precision:} & P_\alpha^c = \Pr(\legi \; | \lefi,\notc)\\
\textit{Recall:}    & R_\alpha^c = \Pr(\lefi,\notc \; | \; \legi)
\end{eqnarray}

Our first result shows how much the precision changes after applying a metacognitive condition.  The argument in this paper characterizes the change in precision using a probabilistic argument and, notably, the result is obtained without any assumptions of independence.

\begin{theorem}[Metacognitive Precision Change]
\label{thm:1}
The following identity holds:
\[
P_\alpha^{c}-P_\alpha=K\times\big (\Pr(\alpha \notin \gt \; | \; \lefi, \c)-(1-P_\alpha)\big),
\]
where $K=\frac{\Pr(\c \; | \; \lefi)}{\Pr(\notc \; | \; \lefi)}$
\end{theorem}

In~\citet{xi24}, an analogous result is shown. Both results suggest finding conditions that attempt to maximize the product of probabilities 
$\Pr(\c \; | \; \lefi)$ and $\Pr(\alpha \notin \gt \; | \; \lefi, \c)$ is desirable for precision improvement by ``erasing'' labels 
-- these correspond to support and confidence in that paper.  
It also turns out that this product has computationally desirable properties, as it is submodular (proven in~\citet{xi24}).  However, we point out that the result of~\citet{xi24} did not frame the preliminaries in terms of probability, and hence it was not clear if there were latent assumptions; this probabilistic interpretation and the corresponding proof (in the appendix at \url{https://arxiv.org/abs/2502.05398}) strengthen the results, as now they clearly do not rely on independence.  Further, we obtain additional insights built on this result discussed throughout this paper, and one notable insight is that we immediately obtain a necessary and  sufficient condition for obtaining improvement in precision:
\begin{eqnarray}
\Pr(\alpha \notin \gt \; | \; \lefi, \c) > 1 - P_\alpha
\end{eqnarray}
We can think of $1-P_\alpha$ as the residual of the model -- how much room it has to improve precision.  We can also think of this result as having a practical application in determining if a condition is invariant.  
For example, we would likely assume that $1-P_\alpha$ increases for out-of-distribution samples, and perhaps we could determine some $\epsilon$ where for some $\d$ that $\Pr(\alpha \notin \gt \; | \; \lefi, \c,\d)$ is within $\epsilon$ of confidence $\Pr(\alpha \notin \gt \; | \; \lefi, \c)$.  So, for example, if we identify the confidence value by examining multiple samples, and obtain an error, we can determine up to which value of $1-P_\alpha$ the condition is expected to be invariant.  Another aspect of this result is that we can also prove that our error-detecting definition (Definition~\ref{def:det}) is an equivalent condition.

\begin{theorem}[Error Detection is Necessary and Sufficient]
\label{thm:edns}
Condition $c$ is error detecting iff $P_\alpha^c  \geq P_\alpha$.
\end{theorem}
\begin{proof}
\noindent$(\rightarrow)$ By way of contradiction, assume $c$ is error detecting and $P_\alpha^c  < P_\alpha$, which gives us:
\begin{small}
\begin{eqnarray*}
&\Pr(\lefi)\Pr(\leg,\lefi,\notc) <\\
&\,\,\,\,\,\,\,\,\,\,\,\Pr(\leg,\lefi)\Pr(\lefi,\notc)
\end{eqnarray*}
\begin{eqnarray*}
&\Pr(\lefi)(\Pr(\leg,\lefi)-\\
&\,\,\,\,\,\,\,\,\,\,\,\Pr(\leg,\lefi,\c)) <\\
&\,\,\,\,\,\,\,\,\,\,\,\Pr(\leg,\lefi)(\Pr(\lefi)-\Pr(\lefi,\c))
\end{eqnarray*}
\begin{eqnarray*}
&\Pr(\lefi)(-\Pr(\leg,\lefi,\c)) <\\
&\,\,\,\,\,\,\,\,\,\,\,\Pr(\leg,\lefi)(-\Pr(\lefi,\c))
\end{eqnarray*}
\begin{eqnarray*}
&\Pr(\leg \; | \; \lefi) < \Pr(\leg|\lefi,\c)
\end{eqnarray*}
\end{small}

However, this contradicts the definition of error detecting.  The remaining part of the proof is in the appendix.
\end{proof}

\noindent\textbf{Characterization of recall.} 
Next, we look at the change in recall. The recall prior to disregarding the classification for a prediction is
$\Pr(\lef \; | \; \leg)$.  Once we apply the condition, the recall for class $i$ becomes $\Pr(\lef,\notc \; | \; \leg)$.
Note that in doing so, we are essentially turning off model predictions, which can only reduce recall.  We can show an equivalent value to the reduction in recall with the following result.  Again, there is an analogous result in~\citet{xi24}; however our proof (in the appendix) presents a probabilistic argument and the proof clearly illustrates that no independence assumptions are made.

\begin{theorem}[Recall Reduction]
\label{thm:recall}
$\Pr(\lef \; | \; \leg) - \Pr(\lef, \notc \; | \; \leg) =$
\begin{scriptsize}
\[
\Pr(\leg \; | \; \lef,\c)\Pr(\c \; | \; \lef)\frac{\Pr(\lef \; | \; \leg)}{\Pr(\leg \; | \; \lef)}.
\]
\end{scriptsize}
\end{theorem}
\begin{proof}
\begin{small}
\begin{eqnarray*}
\Pr(\lef | \leg)- \Pr(\lef, \notc | \leg)\\
=\Pr(\lef,\c | \leg)\\
=\frac{\Pr(\lef,\c,\leg)}{\Pr(\leg)}\times\frac{\Pr(\lef,\c)}{\Pr(\lef,\c)}\\
=\frac{\Pr(\leg|\lef,\c)\times\Pr(\lef,\c)
\times\Pr(\lef)}{\Pr(\leg)\times\Pr(\lef)}\\
=\frac{\Pr(\leg|\lef,\c)\times\Pr(\c|\lef)
\times\Pr(\lef)}{\Pr(\leg)}
\end{eqnarray*}
\end{small}
\end{proof}

Here we see that the primary driver of recall reduction is $\Pr(\leg \; | \; \lef,\c)$, which is the probability of the model obtaining the correct answer under the metacognitive condition.  Likewise, the other term dependent on the conditions that impacts the reduction in recall is $\Pr(\c \; | \; \lef)$, which is the probability of a condition occurring with a prediction.  Note that both the decrease in recall and the increase in precision trend with this quantity. In the next section, on limitations, we shall understand how this quantity can be bounded.

\section{Limitations of Metacognitive Conditions}

We now explore some of the limits of our approach to metacognitive improvement.  Anecdotally, we noticed in prior metacognitive applications of EDCR that often detection seems easier than correction.  
In~\citet{xi24} and~\citet{lee24}, correction was conducted by changing the class label resulting from a condition-class pair that led to an error.  In the notation of this paper, such a precondition would be ``$\lefi,\c$'', meaning the condition of the model predicting class $\alpha$ while condition $c$ is also true.  This allows us to overcome the detection-induced deficit (a consequence of Theorem~\ref{thm:recall}), which was demonstrated empirically in~\citet{lee24} where such metacognitive correction could ensemble rules to directly improve recall over single-model baselines.  
However, it was less effective in improving recall in the experiments of~\citet{xi24}, where the use case consisted of five classes.  To understand why this occurs, consider a model that cannot distinguish between a set of samples, all classified under condition $c$.  A well-trained model assigns class $i$, the most probable class by training, but the probability of $i$ being correct is lower than the average precision for predictions of class $i$.  However, the next most probable class, $j$, is lower still, and picking this would lower overall loss.  As a result, without another condition or something else to distinguish these samples, the model and the metacognitive correction cannot re-assign those samples a new label that improves overall performance, while at the same time effectively identifying a case where the model had an error with a high probability.  Consider the following:

\begin{example}
\label{ex:3}
Consider the scenario from Example~\ref{ex:2}.  Suppose in cases where the error is observed (both classes $\toyota$ and $\us$ are predicted) that despite $\dodge$ being the ``best'' correction of class $\toyota$ we have the following:
\begin{eqnarray*}
\Pr(\dodge \in \fcar \; | \; \toyota \in \fcar, \us \in \fcar)\\
\leq \Pr(\dodge \in \fcar \; | \; \toyota \in \fcar)
\end{eqnarray*}
\end{example}

It turns out that the situation in Example~\ref{ex:3} leads to a reduction in precision for the class by which the label is re-assigned.  
We now formalize this argument.  Intuitively, if the precision for class $j$ conditioned on reclassifying items originally classified as $i$ under condition $c$ is lower, then the overall precision for classification of class $j$ will drop.

\begin{theorem}[Limits of Reclassification]
If we have that $\Pr(\legj \; | \; \lefi,\ci)\leq \Pr(\legj \;| \;\lefj)$, then:
\begin{small}
\[
\Pr(\legj \; | \; \lefj) \geq \Pr(\legj \; | \; \lefj \vee (\lefi,\ci)).
\]
\end{small}
\end{theorem}
\begin{proof}
By way of contradiction, assume:\\[2pt]
\begin{small}
$\Pr(\legj \; | \; \lefj) < \Pr(\legj \; | \; \lefj \vee (\lefi,\ci))$
\begin{eqnarray*}
\Pr(\legj \; | \; \lefj) < \Pr(\legj \; | \; \lefj \vee (\lefi,\ci))\\[4pt]
\frac{\Pr(\lefi,\ci)}{\Pr(\lefj)}<\frac{\Pr(\legj,\lefi,\ci)}{\Pr(\legj,\lefj)}\\[4pt]
\frac{\Pr(\legj,\lefj)}{\Pr(\lefj)}<\frac{\Pr(\legj,\lefi,\ci)}{\Pr(\lefi,\ci)}\\[4pt]
\Pr(\legj \; | \;\lefj)<\Pr(\legj \; | \; \lefi,\ci),
\end{eqnarray*}
\end{small}
which contradicts the statement of the theorem.
\end{proof}
We also note there are other potential limits on reclassification due to EDCR.  Specifically, if rules are learned in a manner that allows for inconsistencies (which would be possible if rules are learned among different models independently) then proper corrective action becomes less clear.  This is an active area of inquiry.

\smallskip
\noindent\textbf{Limitations to precision improvement/recall reduction.}
In another new result from our analysis, we show we can bound the quantity $\Pr(\c \; | \; \lefi)$, which as pointed out earlier can magnify or suppress the amount of change in precision, or amplify the reduction in recall based on the prevalence of the metacognitive condition.  Specifically, this is bounded by $\Pr(\c \; | \;\leng,\lefi)$ when $c$ is error-detecting (shown in the next corollary).  
The practical application of the result, when identifying a metacognitive condition, is that we can understand the power of such a condition by only analyzing a subset of a dataset where the model was correct, which may have implications for large-scale data analytics and imbalanced classification problems.  Note that we never assume $\c$ is independent of $\leng$ or $\lefi$ -- this result stems directly from the assumption that $c$ is error detecting.

\begin{corollary}
If $c$ is error detecting, then:
\[
\Pr(\c \;| \; \lefi) \leq \Pr(\c \;| \;\leng,\lefi) 
\]
\end{corollary}
\begin{proof}
By definition of error detecting, we have:
\begin{eqnarray}
\Pr(\leng \; | \;\lefi\c)\geq \Pr(\leng \; | \; \lefi)\\[4pt]
= \frac{\Pr(\c \; | \;\lefi)\Pr(\leng \; | \; \lefi,\c)}{\Pr(\c \; | \; \leng,\lefi)}.
\end{eqnarray}
This in turn gives us:
\begin{eqnarray}
1\geq  \frac{\Pr(\c \; | \; \lefi)}{\Pr(\c \; | \; \leng,\lefi)}\\
\Pr(\c \; | \; \leng,\lefi) \geq \Pr(\c \; | \; \lefi),
\end{eqnarray}
thus completing the proof.
\end{proof}

\section{Conclusion: \\
Directions for Future Research}

In this paper, we reviewed a hybrid-AI technique for metacognition known as EDCR, and provided a probabilistic argument that supports previous findings.  This also suggests future research directions, so we end the paper with a discussion of some of these potential areas of inquiry.

\smallskip
\noindent\textbf{Role of Domain Knowledge and Inconsistency.}  
In Example~\ref{ex:2} we showed how metacognitive conditions can be used to identify inconsistencies, and the work of~\citet{kri24} demonstrates this empirically by both recovering latent constraints and using that to improve model loss via LTN.  When we view conditions from the standpoint of distribution invariance, logical inconsistency clearly meets these criteria.  This raises an interesting question: can consistency with domain knowledge be used for error correcting?  
Neurosymbolic approaches already use this concept to reduce training loss ~\cite{10637618}. However, understanding how consistency can be used for metacognitive correction remains an open question.

\smallskip
\noindent\textbf{Multi-model / Multi-modal reasoning.}  
The results of this paper suggest that EDCR can be most effective for improving model precision while sacrificing recall (Theorems~\ref{thm:1} and~\ref{thm:recall}).  The results of~\citet{lee24} essentially leverage this property to ensemble models together, as different models can be precise and their combined recall will lead to an increase if they are identifying different phenomena.  The ensembling of multi-modal models through EDCR seems to be a natural fit, as models of different modalities likely exhibit complementary capabilities.  However, work like~\citet{lee24} assumes one mode is the ``base model'' while the others are used to correct it.  The study of EDCR (or other metacognitive methods) to ensemble two or more models without a designed ``base model'' remains an open question.

\smallskip
\noindent\textbf{Online metacognition and data efficiency.}  Quantities such as $\Pr(\legi \; | \lefi, \c)$ and $\Pr(\c \; | \;  \lefi)$ (which are analogous to confidence and support) are typical byproducts of metacognitive rule learning under EDCR, and as seen in this paper are key quantities for validating invariance of conditions (Theorem~\ref{thm:edns}).  Likewise, an experiment in~\citet{xi24} shows how rules can be learned from a new distribution that are effective with a small portion of data.  Together, can these results suggest rapid adoption to a new distribution of data by identifying conditions on the fly (i.e., online learning of metacognitive conditions)?

\section*{Acknowledgments}
This research was supported by the Defense Advanced Research Projects Agency (DARPA) under Cooperative Agreement No. HR00112420370, the U.S. Army Combat Capabilities Development Command (DEVCOM) Army Research Office under Grant No. W911NF-24-1-0007, and the U.S. Army DEVCOM Army Research Lab under Support Agreement No. USMA 21050. The views expressed in this paper are those of the authors and do not reflect the official policy or position of the U.S. Military Academy, the U.S. Army, the U.S. Department of Defense, or the U.S. Government.


\small

\appendix
\section{Full Proof of Theorem~\ref{thm:1}}
\noindent\textbf{Claim 1:} \[
P_\alpha^{c}=\frac{P_\alpha-\Pr(\leg | \lefi, \c)\Pr(\c | \lefi)}{\Pr(\notc | \lefi)}
\]
\noindent\textbf{Proof of Claim 1:}
\begin{eqnarray*}
    P_\alpha^{c}=&\Pr(\alpha \in \gt | \alpha \in f_i,  c\notin \m)\\
    =&\frac{\Pr(\alpha \in \gt , \alpha \in f_i,  c\notin \m)}{\Pr(\alpha \in f_i,  c\notin \m)}\\
    =&\frac{\Pr(\alpha \in \gt , \alpha \in f_i,  c\notin \m)}{\Pr(\alpha \in f_i,  c\notin \m)}\frac{\Pr(\lefi,\notc)}{\Pr(\lefi)\Pr(\notc | \lefi)}\\
    =&\frac{\Pr(\alpha \in \gt, \notc | \lefi)}{\Pr(\notc | \lefi)}\\
    =&\frac{P_\alpha-\Pr(\alpha \in \gt, \c | \lefi)}{\Pr(\notc | \lefi)}\\
    =&\frac{P_\alpha}{\Pr(\notc | \lefi)}-\frac{\Pr(\alpha \in \gt, \c, \lefi)}{\Pr(\lefi)\Pr(\notc | \lefi)}\\
    =&\frac{P_\alpha}{\Pr(\notc | \lefi)}-\frac{\Pr(\alpha \in \gt, \c, \lefi)}{\Pr(\lefi)\Pr(\notc | \lefi)}\frac{\Pr(\c,\lefi)}{\Pr(\c,\lefi)}\\
    =&\frac{P_\alpha}{\Pr(\notc | \lefi)}-\frac{\Pr(\alpha \in \gt|  \lefi,\c)}{\Pr(\lefi)\Pr(\notc | \lefi)}\frac{\Pr(\c,\lefi)}{1}\\
    =&\frac{P_\alpha}{\Pr(\notc | \lefi)}-\frac{\Pr(\alpha \in \gt|  \lefi,\c)\Pr(\c|\lefi)}{\Pr(\notc | \lefi)}\\
    =&\frac{P_\alpha-\Pr(\alpha \in \gt|  \lefi,\c)\Pr(\c|\lefi)}{\Pr(\notc | \lefi)} 
\end{eqnarray*}
\noindent\textbf{Claim 2:} 
\begin{scriptsize}
\[
P_\alpha^{c}-P_\alpha=(\Pr(i \notin \gt | \lefi, \c)-R_i)\frac{\Pr(\c | \lefi)}{\Pr(\notc | \lefi)}
\]
\end{scriptsize}
\noindent\textbf{Proof of Claim 2:}
$P_\alpha^{c}-P_\alpha=$
\begin{eqnarray*}
&\frac{P_\alpha-\Pr(\leg | \lefi, \c)\Pr(\c | \lefi)}{\Pr(\notc | \lefi)}-P_\alpha\\
=&\frac{P_\alpha-\Pr(\leg | \lefi, \c)\Pr(\c | \lefi)-P_\alpha\Pr(\notc | \lefi)}{\Pr(\notc | \lefi)}\\
=&\frac{P_\alpha-\Pr(\leg | \lefi, \c)\Pr(\c | \lefi)-P_\alpha(1-\Pr(\c | \lefi))}{\Pr(\notc | \lefi)}\\
=&\frac{P_\alpha-\Pr(\leg | \lefi, \c)\Pr(\c | \lefi)+P_\alpha\Pr(\c | \lefi)-P_\alpha}{\Pr(\notc | \lefi)}\\
=&\frac{-\Pr(\leg | \lefi, \c)\Pr(\c | \lefi)+P_\alpha\Pr(\c | \lefi)}{\Pr(\notc | \lefi)}\\
=&(-\Pr(\leg | \lefi, \c)+P_\alpha)\frac{\Pr(\c | \lefi)}{\Pr(\notc | \lefi)}\\
=&(\Pr(i \notin \gt | \lefi, \c)+P_\alpha-1)\frac{\Pr(\c | \lefi)}{\Pr(\notc | \lefi)}\\
=&(\Pr(i \notin \gt | \lefi, \c)-R_i)\frac{\Pr(\c | \lefi)}{\Pr(\notc | \lefi)}\\
\end{eqnarray*}

\section{Second claim of Theorem~\ref{thm:edns}}
\noindent$(\leftarrow)$ By way of contradiction, assume that $P_\alpha^c \geq P_\alpha$ and that $c$ is not error detecting.
\begin{scriptsize}
\begin{eqnarray*}
\Pr(\legi|\lefi,\c)> \Pr(\legi|\lefi)\\
\frac{\Pr(\legi,\lefi,\c)}{\Pr(\lefi,\c)}>\frac{\Pr(\legi,\lefi)}{\Pr(\lefi)}\\
\frac{\Pr(\legi,\lefi)-\Pr(\legi,\lefi,\notc)}{\Pr(\lefi)-\Pr(\lefi,\notc)}>\frac{\Pr(\legi,\lefi)}{\Pr(\lefi)}\\
\frac{-\Pr(\legi,\lefi,\notc)}{\Pr(\legi,\lefi)}>\frac{-\Pr(\lefi,\notc)}{\Pr(\lefi)}\\
\frac{\Pr(\legi,\lefi,\notc)}{\Pr(\legi,\lefi)}<\frac{\Pr(\lefi,\notc)}{\Pr(\lefi)}\\
\frac{\Pr(\legi,\lefi,\notc)}{\Pr(\lefi,\notc)}<\frac{\Pr(\legi,\lefi)}{\Pr(\lefi)}\\
\Pr(\legi|\lefi,\notc)<\Pr(\legi|\lefi)
\end{eqnarray*}
\end{scriptsize}

However, this contradicts $P_\alpha^c \geq P_\alpha$, thus completing the proof.


\end{document}